\documentclass{article}
\usepackage{ijcai22}
\usepackage{booktabs}
\usepackage{multirow}
\usepackage{times}
\usepackage{soul}
\usepackage{url}
\usepackage[hidelinks]{hyperref}
\usepackage[utf8]{inputenc}
\usepackage[small]{caption}
\usepackage{graphicx}
\usepackage{amsmath}
\usepackage{amsthm}
\usepackage{amssymb}
\usepackage{color}
\usepackage{algorithm}
\usepackage{algorithmic}
\usepackage{dsfont}
\usepackage{mathtools}
\usepackage[normalem]{ulem}
\usepackage{makecell}

\usepackage[table,xcdraw]{xcolor}

\newtheorem{theorem}{Theorem}

\newtheorem{lemma}{Lemma}

\theoremstyle{definition}
\newtheorem{definition}{Definition}
\newtheorem{assumption}{Assumption}
\newtheorem{remark}{Remark}

\DeclareMathOperator{\VCdim}{VC}

\DeclareMathOperator{\loss}{\mathcal{\ell}_{0-1}}
\DeclareMathOperator*{\argmax}{arg\,max}

\DeclareMathOperator{\define}{\coloneqq}
\DeclareMathOperator{\F}{\mathcal{F}}

\title{Learnability of Competitive Threshold Models}

\author{
Yifan Wang
\And
Guangmo Tong\\
\affiliations
Department of Computer and Information Sciences, University of Delaware, USA\\
\emails
\{yifanw, amotong\}@udel.edu
}

\begin{document}
\maketitle

\begin{abstract}
Modeling the spread of social contagions is central to various applications in social computing. In this paper, we study the learnability of the competitive threshold model from a theoretical perspective. We demonstrate how competitive threshold models can be seamlessly simulated by artificial neural networks with finite VC dimensions, which enables analytical sample complexity and generalization bounds. Based on the proposed hypothesis space, we design efficient algorithms under the empirical risk minimization scheme. The theoretical insights are finally translated into practical and explainable modeling methods, the effectiveness of which is verified through a sanity check over a few synthetic and real datasets. The experimental results promisingly show that our method enjoys a decent performance without using excessive data points, outperforming off-the-shelf methods.

\end{abstract}

\section{Introduction}

Social contagion phenomena, such as the propagation of behavior patterns or the adoption of new technologies, have attracted huge research interests in various fields (e.g., sociology, psychology, epidemiology, and network science \cite{cozzo2013contact,Goldstone2005ComputationalMO,camacho2020four}). Formal methods for modeling social contagion are crucial to many applications, for example, the recommendation and advertising in viral marketing \cite{kempe2003maximizing,tong2018coupon} and misinformation detection \cite{budak2011limiting,tong2020stratlearner}. Given the operational diffusion models, there is a fundamental question regarding the learnability of the models: to which extent can we learn such models from data?. In this paper, we study the classic threshold models in which the core hypothesis is that the total influence from active friends determines the contagion per a threshold. We seek to derive its PAC learnability and sample complexity by which efficient learning algorithms are possible.

\paragraph{Motivation and related works.} Learning contagion models from data has been a central topic in social computing. One research branch focuses on estimating the total influence resulting from an initial status \cite{li2017deepcas,gomez2016influence,tang2009social}; in contrast, our learning task attempts to predict the diffusion status for each node. For node status prediction, it has been proved that deep learning methods are promising when rich features are available (e.g., \cite{qiu2018deepinf,leung2019personalized}), but their generalization performance in theory is often unclear. Recent works have demonstrated the sample complexity of influence learning for single-cascade models \cite{he2016learning,du2014influence}, where certain technical
conditions are required and the approach therein is different from ours in a substantial way. Our research is inspired by the recent works that show the possibility of simulating single-cascade contagion models through neural networks \cite{narasimhan2015learnability,adiga2019pac}. Departing from the common surrogate modeling methods where the learned mappings are less transparent, the pursued idea seeks to design neural networks with their layer structures mimicking real networks. In this paper, we take a step towards the case of general threshold models, without limiting the number of cascades. Comparing to the single-cascade case where threshold layers are sufficient, multi-cascade diffusion requires more sophisticated designs in order to properly propagate influence summation between layers without losing the information of cascade identity.

\paragraph{Contribution.} Our work can be summarized as follows.

\begin{itemize}
    \item \textbf{Model design.} For a finite-precision system, we present elementary designs showing that the general threshold model can be simulated explicitly by neural networks with piecewise polynomial activation functions. 
    \item \textbf{Theoretical analysis.} 
    Based on our hypothesis design, we prove that the general threshold model is PAC learnable, and design efficient empirical risk minimization schemes with a provable sample complexity.
    \item \textbf{Empirical studies.} We experimentally examine the proposed methods over synthetic and real datasets. The results suggest that methods based on explicit simulation outperform off-the-shelf learning methods by an evidence margin. They also verify that our methods work in the way it supposed to.
\end{itemize}
The proofs, full experimental results, and source code are provided in the supplementary material\footnote{\url{https://github.com/cdslabamotong/LTInfLearning}}.

\newpage
\section{Preliminaries}

\subsection{Competitive linear threshold model}
We follow the standard competitive linear threshold (CLT) model \cite{borodin2010threshold}, which has been widely adopted in the formal study of social contagions \cite{he2012influence,tzoumas2012game,bozorgi2017community}. A social network is represented as a directed graph $G=(V, E)$, where $V=\{v_1,...,v_{N}\}$ is the set of nodes and $E$ is the set of edges, with $N = |V|$ and $M = |E|$ being the numbers of nodes and edges, respectively. For each node $v \in V$, we denote by $N(v)\subseteq V$ the set of its in-neighbors. 

We consider the scenario where there are $S\in \mathbb{Z}^+$ information cascades, each of which starts to spread from their seed set $\psi_0^s \subseteq V$ for $s \in [S]$. Associated with each node $v_i$ and each cascade $s$, there is a threshold $\theta_i^s \in [0, 1]$; associated with each edge $(v_i, v_j)\in E$ and each cascade $s$, there is a weight $w_{i,j}^s \in [0, 1]$. Without loss of generality, we assume that the seed sets are disjoint and the weights are normalized such that $\sum_{v_j \in N(v_i)} w_{(j,i)}^{s}\leq 1$ for each node $v_i \in V$ and cascade $s$. The nodes are initially inactive and can be $s$-active if activated by cascade $s$. In particular, the diffusion process unfolds step by step as follows:

\begin{itemize}

    \item \textbf{Step $\mathbf{0}$:}  For each cascade $s \in [S]$, nodes in $\psi_0^s$ become $s$-active.

    \item \textbf{Step $\mathbf{t>0}$:} Let $\psi_{t-1}^s \subseteq V$ be the set of $s$-active nodes after $t-1$ step. There are two phases in step $t$.  
    
    \begin{itemize}
        \item \textbf{Phase 1}: For an inactive node $v_i \in V$, let $\zeta^s_i$ be the summation of the weights from its $s$-active in-neighbors:
        \begin{align}
        \zeta^s_i= \sum_{v_j \in N(v_i) \bigcap \psi^{s}_{t-1}} w_{j,i}^{s}.
        \end{align}
        The node $v_i$ is activated by cascade $s$ if and only if $\zeta_i^s \geq \theta_{i}^{s}$. After phase 1, it is possible that a node can be activated by more than one cascades.
        
        \item \textbf{Phase 2}: For each node $v_i$, it will be finally activated by the cascade $s^*$ with the largest weight summation:
        \begin{align}
        s^* = \argmax _{s: \zeta^i_s \geq \theta^i_s} \zeta_i^s.    
        \end{align}
        For tie-breaking, cascades with small indices are preferred.

    \end{itemize}

\end{itemize}
\begin{figure}
    \centering
    \includegraphics[width = 0.45\textwidth]{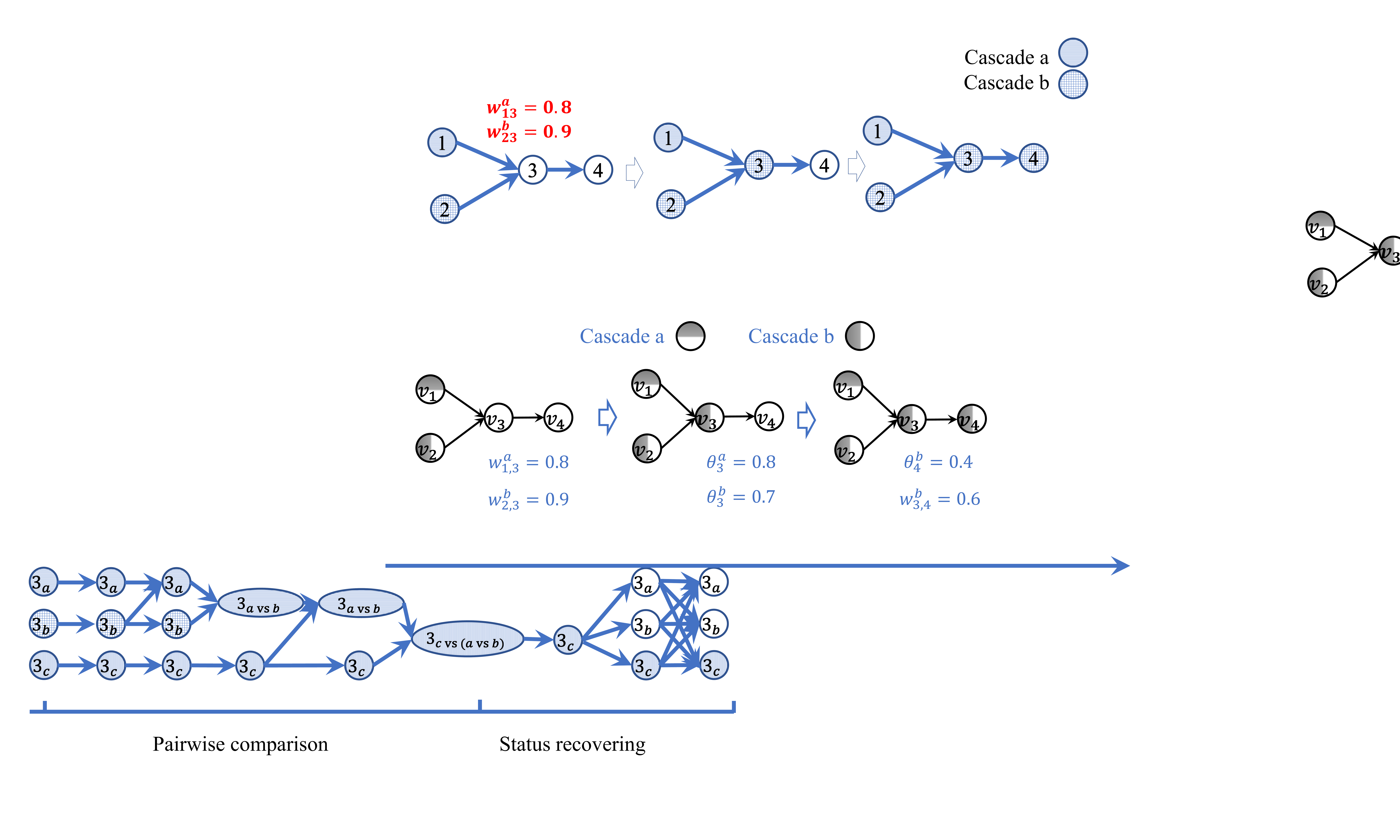}
    \caption{Diffusion process example. In this example, we have two cascades with seed sets $\{v_2\}$ and $\{v_1\}$. The weights and thresholds are given as graph shows. According to the CLT model, $v_3$ will become $1$-active after one time step. $v_4$ will become $1$-active at the end of diffusion process.}
    \label{fig:CLT model}
    \vspace{-2mm}
\end{figure}

Clearly, there are at most $n$ diffusion steps, and without loss of generality, we may assume that there are always $n$ diffusion steps. An illustration is given in Figure \ref{fig:CLT model}. The following notations are useful for formal discussions.

\begin{definition}

We use a binary matrix $\mathbf{I}_{0} \in \{0, 1\}^{N\times S}$ to denote the initial status, where $\mathbf{I}_{0, i, s}=1$ if and only if node $v_i$ is in the seed set $\psi_0^s$ of cascade $s$. For a time step $t>0$, the diffusion status is denoted by a binary tensor $\mathbf{I}_{t} \in \{0,1\}^{N \times S \times 2}$, where, for a cascade $s$, $\mathbf{I}_{t, i, s, p} = 1$ if and only if node $v_i$ is $s$-active after phase $p \in \{1,2\}$ in time step $t$.

\end{definition}

We are interested in the learnability of the influence function of CLT models.

\begin{definition}
Given a CLT model, the influence function $F: \{0,1\}^{N \times S} \mapsto \{0,1\}^{N \times S}$ maps from the initial status $\mathbf{I}_0$ to the final status $\mathbf{I}_{N,:,:,2}$.
\end{definition}

\subsection{Learning settings}

Assuming that the social graph $G$ is unknown to us, we seek to learn the influence function using a collection $D$ of $K\in \mathbb{Z}^+$ pairs of initial status and corresponding diffusion statuses:
\begin{align}
D=\Big\{(\mathbf{I}_0^1, \mathbf{I}_{1:N, :, :, 2}^1),...,(\mathbf{I}_0^K, \mathbf{I}_{1: N, :, :, 1: 2}^K)\Big\}
\end{align}
where $\mathbf{I}_{1:N, :, :, 1:2}^k$ denotes the diffusion status after each phase in each diffusion step:
$$\mathbf{I}_{1:N, :, :, 2}^k=\Big[\mathbf{I}_{1,:,:,1}^k, \mathbf{I}_{1,:,:,2}^k,\mathbf{I}_{2,:,:,1}^k, \mathbf{I}_{2,:,:,2}^k,...,\mathbf{I}_{N,:,:,1}^k, \mathbf{I}_{N,:,:,2}^k\Big].$$ 
We use the zero-one loss $\loss$ to measure the difference between the prediction $\hat{\mathbf{I}}_{N,:,:,2}$ and the ground truth $\mathbf{I}_{N,:, :, 2}$, which is defined as
\begin{flalign*}
&\loss(\mathbf{I}_{N, :, :, 2}, \hat{\mathbf{I}}_{N,:, :, 2}) \define \hspace{4.5cm} \\
    &\hspace{2cm}\frac{1}{2NS} \sum _{i \in [N]}\sum_{s \in [S]} \mathds{1}\big [\mathbf{I}_{N,i,s,2} \neq \hat{\mathbf{I}}_{N,i,s,2}\big]
\end{flalign*}
For a distribution $\mathcal{D}$ over the input $\mathbf{I}_0$ and a CLT model generating the training data $D$, we wish to leverage $D$ to learn a mapping $\hat{F}: \{0,1\}^{N \times S} \mapsto \{0,1\}^{N \times S}$ such that the generalization loss can be minimized:
\begin{equation}\label{equ:gerror}
L(\hat{F}) \define \mathbb{E}_{\mathbf{I}_0  \sim \mathcal{D} }\Big[\loss\big({F}(\mathbf{I}_0), \hat{F}(\mathbf{I}_0)\big)\Big].
\end{equation}

\section{Learnability analysis}

In this section, we present a PAC learnability analysis for the influence function. The following assumptions will be made throughout the analysis. 

\begin{assumption}{\label{assum: finite}}
Following the conventions regarding finite precision system \cite{holi1993finite,weiss2018practical,dembo1989complexity}, we assume that the diffusion model is a decimal system of a finite precision: for the involved parameters (i.e., $w_{i,j}^s$ and $\theta_i^s$), number of their decimal places is upper bounded by a constant $Q \in \mathbb{Z}^+$. 
\end{assumption}

\begin{assumption}{\label{s}}
The cascade number $S = 2^Z$ is a power of $2$ for some constant $Z\in \mathbb{Z}^+$. 

\end{assumption} 

\begin{figure*}[!t]
    \centering
    \includegraphics[width = 0.98\textwidth]{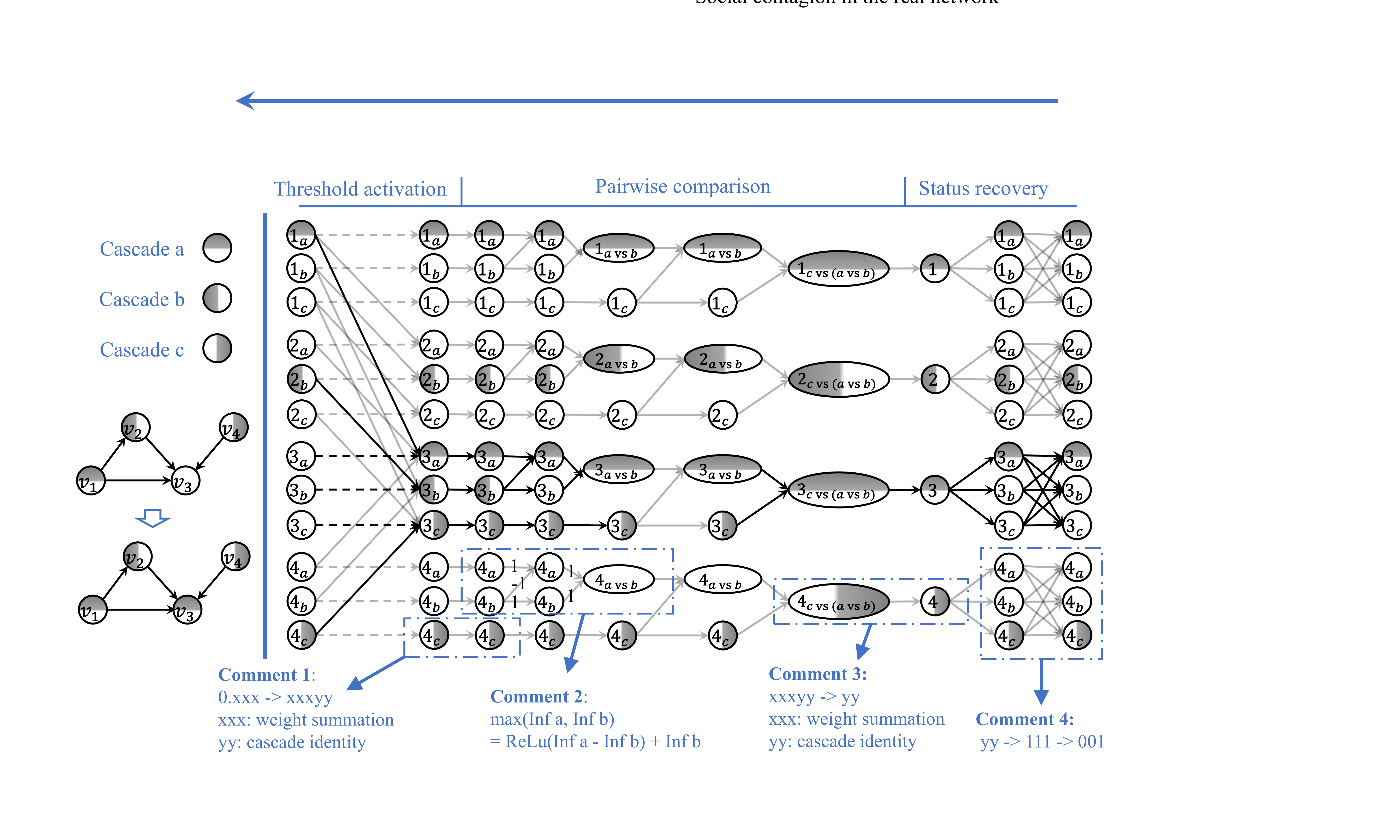}    \caption{\textbf{One-step Simulation.} The left part depicts the scenario where the node $v_3$ is activated by three cascades among which cascade $a$ finally wins the competition. The right parts demonstrates how to simulate such a scenario through three our design; the computation path for node $v_3$ is highlighted. }
    \label{fig:one_step}
    \vspace{-4mm}
\end{figure*}

The overall idea is to design a realizable hypothesis with a finite VC dimension by which the PAC learnability can be established. In order to obtain a realizable hypothesis space, we seek to explicitly simulate the diffusion function through neural networks. While the universal approximation theory ensures that such neural networks exist, we wish for elementary designs of a polynomial size so as to admit efficient learning algorithms. To this end, we first demonstrate how to simulate the one-step diffusion under the CLT model, which can be used repeatedly to simulate the entire diffusion. 
\subsection{Realizable hypothesis space}


For the one-step diffusion, its simulation is done through two groups of layers that are used in turn to simulate phase 1 and phase 2. As illustrated in Figure \ref{fig:one_step}, the simulation starts by receiving the diffusion status $\mathbf{H}^{1}\in \{0, 1\}^{N \cdot S}$ after the last step, where for $i \in [N]$ and $s \in [S]$, $\mathbf{H}^{1}_{(i-1)\cdot S+s} =1$ if and only if node $v_i$ is $s$-active. After $2Z +4$ layers of transformations, it ends by outputting a vector $\mathbf{H}^{2Z+6}\in \{0, 1\}^{N \cdot S}$, which corresponds exactly to the diffusion status after one diffusion step under the CLT model. The layer transformations all follow the canonical form:
\begin{equation}
    \mathbf{H}^{l+1} = \sigma^{l+1}(\mathbf{W}^{l}\mathbf{H}^{l})
\end{equation}
with $\mathbf{W}^l$ being the weight matrix and $\sigma^{l}$ being a collection of activation functions over the elements of $\mathbf{W}^{l}\mathbf{H}^{l}$. In the rest of this subsection, we will present our designs of the weight matrices and the activation functions that can fulfill the goal of simulation.

\subsubsection{Phase 1}
The single-cascade activation can be simulated by
\begin{align}
\label{eq: H1_H2}
\mathbf{H}^{2} = \sigma^{1}(\mathbf{W}^{1}\mathbf{H}^{1}),  
\end{align}
where for each $i, j \in [N]$ and $s_1, s_2 \in [S]$, we have
\begin{flalign*}
&\mathbf{W}^{1}_{(i-1)\cdot S+s_1, (j-1)\cdot S+s_2} \define \hspace{4.5cm} \\
    &\hspace{2cm}\begin{cases} 
      w_{i,j}^{s_1} & \text{if $(v_i, v_j)\in E$ and $s_1=s_2$} \\
      1 & \text{if $i=j$ and $s_1=s_2$}\\
      0 & \text{otherwise} 
    \end{cases},
\end{flalign*}
and for each $i \in [N]$ and $s\in [S]$, we have 
\begin{equation}
    \sigma ^{1}_{(i-1)\cdot S+s}(x) \define
    \begin{cases} 
      x & \text{if $x\geq\theta_{i}^{s}$} \\
      0 & \text{Otherwise} 
    \end{cases}.
    \end{equation}
For each  $i \in [N]$ and $s\in [S]$, the resulting $\mathbf{H}^{2}_{(i-1)\cdot S+s}$ is equal to the influence summation of cascade $s$ at node $v_i$ if $v_i$ is activated by cascade $s$, and otherwise zero.        
 
\subsubsection{Phase 2}
In simulating phase 2, the general goal is to figure out among the candidate cascades which one wins the competition. Intuitively, this can be implemented through a hierarchy of pairwise comparisons. One technical challenge is that comparison made directly by linear functions tells only the difference between the summations, while we need to keep the summation for future comparisons; we observe that this problem can be solved by using two consecutive layers of piecewise linear functions. A more serious problem is that simply passing the largest influence summation to subsequent layers incurs the loss of the information about the cascade identity (i.e., index), making it impossible to identify the node status -- which is required as the input to simulating the next diffusion step. We propose to address such a challenge by appending the cascade identity to the weight summation as the lowest bits. The total shifted bits are determined by the system precision (Assumption \ref{assum: finite}). Doing so does not affect the comparison result while allowing us to retrieve the cascade identity by the scale and modulo operations.  In particular, two groups of transformations are employed to achieve the pairwise comparison and status recovery. 

\textbf{Pairwise comparison.} With the input $\mathbf{H}^2$ from phase 1, we encode the cascade identity through
\begin{align}
\mathbf{H}^{3} = \sigma^{2}(\mathbf{W}^{2}\mathbf{H}^{2}),  
\end{align}
where for each $i \in [N\cdot S]$ and $j \in [N \cdot S]$, we have
\begin{equation}
\mathbf{W}^2_{i, j} \define
\begin{cases} 
10^{(Q+  \lfloor{\log_{10} S}  \rfloor+1)} & \text{if $i = j$} \\
0 & \text{otherwise} 
\end{cases},
\end{equation}
and for each $i \in [N]$ and $s \in [S]$, we have 
\begin{equation}
\sigma ^{2}_{(i-1)\cdot S+s}(x) = x + s.
\end{equation}
The design of $\mathbf{W}^2$ ensures that sufficient bits are available at the low side, and the activation function $\sigma ^{2}$ then writes the cascade identity into the empty bits (Comment 1 in Figure \ref{fig:one_step}). 

Using the information in $\mathbf{H}^3$, for a node $v_i \in V$, the largest influence summation can be determined by the sub-array 
$$\Big[\mathbf{H}^3_{(i-1)\cdot S+1}, \mathbf{H}^3_{(i-1)\cdot S+2},..., \mathbf{H}^3_{(i-1)\cdot S+S}\Big]$$ through pairwise comparisons. For the comparison between two cascades $s_1$ and $s_2 $  at node $v_i$ ($s_1 < s_2$), as illustrated in Comment 2 in Figure \ref{fig:one_step}, the difference $$\mathbf{H}^3_{(i-1)\cdot S+s_1}-\mathbf{H}^3_{(i-1)\cdot S+s_2}$$ is first fed into a ReLu function of which the result is then added by $\mathbf{H}^3_{(i-1)\cdot S+s_2}$. We can verify that this returns exactly $$\max(\mathbf{H}^3_{(i-1)\cdot S+s_1}, \mathbf{H}^3_{(i-1)\cdot S+s_2}).$$ Therefore, two layers are used to eliminate a half of the candidate cascades, thereby resulting in $2Z$ layers in total. The output of this part is $\mathbf{H}^{2Z+3} \in \mathbb{R}^N$ in which we have
\begin{align*}
\mathbf{H}^{2Z+3}_i = \max\Big(\mathbf{H}^3_{(i-1)\cdot S+1}, \mathbf{H}^3_{(i-1)\cdot S+2},..., \mathbf{H}^3_{(i-1)\cdot S+S}\Big)
\end{align*}
for each $i \in [N]$. 

Notably, the lowest $Q$ bits in $\mathbf{H}^{2Z+3}$ stores the cascade index, which can be recovered by 
\begin{align}
\mathbf{H}^{2Z+4} = \sigma^{2Z+3}(\mathbf{W}^{2Z+3}\mathbf{H}^{2Z+3}),  
\end{align}
where $\mathbf{W}^{2Z+3}$ is the identity matrix and we have
\begin{equation}
\sigma^{2Z+3}(x) \colon= x  \mod 10^{ \lfloor{\log_{10} S}  \rfloor+1},
\end{equation}
which is shown in Comment 3 in Figure \ref{fig:one_step}. 

Finally, the cascade indices in $\mathbf{H}^{2Z+4}$ are converted to binary indicators through two layers. The first layer transforms $\mathbf{H}^{2Z+4}$ into a binary vector in $\{0,1\}^{N\cdot S}$ by
\begin{align}
\mathbf{H}^{2Z+5} = \sigma^{2Z+4}(\mathbf{W}^{2Z+4}\mathbf{H}^{2Z+4}),  
\end{align}
where for each $i, j \in [N]$ and $s \in [S]$, we have
\begin{equation}
\mathbf{W}^{2Z+4}_{i, (j-1)\cdot S+s} \define
\begin{cases} 
 1 & \text{if $i=j$} \\
0 & \text{otherwise} 
\end{cases},
\end{equation}
and 
\begin{equation}
\sigma^{2Z+4}_{(i-1)\cdot S+s}(x) \define 
\begin{cases} 
1 & \text{if $x \geq s$} \\
0 & \text{otherwise} 
\end{cases}.
\end{equation}
This ensures that in each group $$\Big[\mathbf{H}^{2Z+5}_{(i-1)\cdot S+1},\mathbf{H}^{2Z+5}_{(i-1)\cdot S+2}, ...,\mathbf{H}^{2Z+5}_{(i-1)\cdot S+S}\Big],$$ $\mathbf{H}^{2Z+5}_{(i-1)\cdot S+s}$ is $1$ if and only if $s$ is no larger than the index of the cascade that activates $v_i$. Finally, $\mathbf{H}^{2Z+5}$ is transformed into the form that matches exactly the diffusion status after ones-step diffusion, which can be achieved through
\begin{align}
\mathbf{H}^{2Z+6} = \sigma^{2Z+5}(\mathbf{W}^{2Z+5}\mathbf{H}^{2Z+5}),  
\end{align}
where for each $i, j \in [N]$ and $s_1, s_2 \in [S]$, we have
\begin{equation}
\mathbf{W}^{2Z+5}_{(i-1)\cdot S+s_1, (j-1)\cdot S+s_2} \define
\begin{cases} 
 1 & \text{if $i=j$ and $s_1 \leq s_2$} \\
  -1 & \text{if $i=j$ and $s_1 = s_2$} \\
0 & \text{otherwise} 
\end{cases},
\end{equation}
and 
\begin{equation}
\sigma^{2Z+5}_{(i-1)\cdot S+s}(x) \define 
\begin{cases} 
1 & \text{if $x\geq s$} \\
0 & \text{otherwise} 
\end{cases}.
\end{equation}
The last two layers are illustrated in Comment 4 in Figure \ref{fig:one_step}.

\begin{lemma}{\label{lemma:clt_to_nn}}
One-step diffusion in a CLT model can be simulated by a feed-forward neural network composed of $O(Z)$ layers, with $O(S (N+M))$ adjustable weights and $O(N \cdot S)$ piecewise linear computation units each of which with $O(2^Q)$ pieces.
\end{lemma}

\begin{definition}[Hypothesis Space]
Repeating such one-step simulations for $N$ steps, the CLT model can be explicitly simulated. Taking the weights $w_{i,j}^s$ and the thresholds $\theta_i^s$ as parameters, we denote by $\F$ the hypothesis space formed by such neural networks.  
\end{definition}

One important property is that the entire neural network is composed of piecewise polynomial activation functions, which will be used in deriving its sample complexity. 


\subsection{Efficient ERM}
Theorem \ref{lemma:clt_to_nn} suggests that for any sample set $D$, there always exists a perfect hypothesis in $\F$. In what follows, we show that such an empirical risk minimization solution can be efficiently computed. It is sufficient to find the parameters that can ensure the output of each activation function coincides with the diffusion result. Since the activation functions are all pairwise linear functions, the searching process can be done by linear programming with a polynomial number of constraints. Formally, we have the following statement.

\begin{lemma}
\label{lemma: efficient_erm}
For a CLT model and each sample set $D$, there exists a hypothesis in $\F$ with zero training error, and it can be computed in polynomial time in terms of $N, M, S$, and $|D|$.

\end{lemma}

\subsection{Generalization Performance}
To establish the learnability of class $\F$, we first analyze the VC dimension of the binary-valued function class $\F_{i,j}$ that is obtained by restricting the output of to the diffusion status of node $v_i$ regarding cascade $s$ in the last layer of the neural network. That is,
\begin{align*}
\F_{i,s}=\Big\{ f_{i,s}: i \in [N], s \in [S],  f_{i,s}(\mathbf{H}^1)=\mathbf{H}^{1+N(2Z+5)}_{i\cdot S+s} \Big\}.
\end{align*}

\begin{lemma}{\label{lemma:VC}}
For each $i \in [N]$ and $s \in [S]$, the VC dimension of class $\F_{i,s}$ is $\VCdim(\mathcal{F}_{i,s}) = \Tilde{O}(S\cdot Q\cdot (N+M))$.
\end{lemma}

\begin{proof}[proof sketch]
The proof is obtained by first showing that the functions in $\F_{i,s}$ have a manageable VC dimension when restricted to the one-step simulation. Given the fact that repeating such one-step simulations does not increase the model capacity in terms of shattering points, the VC dimension of the entire class $\F_{i,s}$ can be derived immediately.
\end{proof}

\begin{remark}
\label{remark: special_case}
The above suggests that the VC dimension is linear in the network size, the number of cascades, and the coding size of the numerical values. When there is only one cascade, pairwise comparisons and status recovery are no longer needed, and therefore, the size of the neural network does not depend on the system precision (i.e., $Q$). This gives a VC dimension of $\Tilde{O}(N+M)$, which recovers the results in \cite{narasimhan2015learnability}. Interestingly, similar results apply to the case of $S=2$ but not to larger $S$. We provide more details in Appendix \ref{subsec: remark_special_case}.
\end{remark}

With Lemmas \ref{lemma: efficient_erm} and \ref{lemma:VC}, the generalization bound of the ERM scheme for predicting the status of one node regarding one cascade can be derived through the class result of statistical learning theory, and the sample complexity of $\F$ follows from the union bound over all the nodes and cascades.

\begin{theorem}{\label{theorem: final}}
The influence function class $\mathcal{F}$ is PAC learnable and the sample complexity for achieving a generalization error no larger than $\epsilon$ with probability at least $1-\delta$ is $ O(\frac{\VCdim(\F_{i,s})\log(1/\epsilon) + \log(NS/\delta)}{\epsilon})$.

\end{theorem}

\section{Experiment}
In this section, we present empirical studies for experimentally examining the practical utility of the proposed method. 
\subsection{Experimental settings}

\textbf{Graphs.} We adopt two classic graph structures (\text{Kronecker} \cite{leskovec2010kronecker} and power-law) and one real graph Higgs, which was built after monitoring the spreading process of the messages posted between 1st and 7th July 2012 regarding the discovery of a new particle with the features of the elusive Higgs boson \cite{de2013anatomy}. The graph statistics are provided in the supplementary material.

\textbf{Diffusion Model.} We simulate the case with the number of cascades $S$ selected from $\{2, 3, 4\}$. The thresholds $\theta_i^s$ are generated uniformly at random from $[0,1]$. We consider two settings for weight generation, following the weighted cascade setting \cite{borodin2010threshold}, the weights on Kronecker and Power-law are set as $w_{i,j}^s = \frac{1}{|N(v_j)| + s/7}$ for each pair $(v_i,v_j) \in E$ and each $s \in [S]$. On Twitter, the weights are sampled uniformly at random from $[0,1]$ and then normalized such that $\sum_{v_i \in N(v_j)}w_{i,j}^s =1 $ for each $v_j \in V$. The parameters are set to be numbers with three decimal places.

\textbf{Samples.} To generate one sample of $(\mathbf{I}_0^1, \mathbf{I}_{1:N, :, :, 2}^1)$, the total number of seed nodes $|\bigcup_{s=1}^S \psi_0^s|$ is generated from $[0.1N, 0.5N]$ uniformly at random, and then a subset of $|\bigcup_{s=1}^S \psi_0^s|$ nodes are randomly selected from $V$. For each selected node, it is assigned to one cascade uniformly at random. Given the initial status, the ground truth is generated following the diffusion model. For each graph, we generate a pool of 2,000 samples among which the average diffusion step is four and the average number of total active nodes is 254.

\textbf{Methods.} Given the samples, the parameters of our model are computed by the linear programming solver Cplex \cite{cplex2009v12}. Alternatively, one can use off-the-shelf supervised learning methods to train a one-step diffusion model on each node, and make predictions by repeatedly simulating the one-step diffusion process. Given the nature of binary classification, we implement three methods, logistic regression (\textbf{LR}), support vector machine (\textbf{SVM}), and multilayer perceptron (\textbf{MLP}), using scikit-learn \cite{scikit-learn}. For such baselines, the number of repetitions for testing is selected from $\{1,2,3,4,5\}$. We also implement the random method that makes a random prediction.

\textbf{Training and Testing.} The size of the training set is selected from $\{50,100,500\}$ and the testing size is $500$. In each run, the training and testing samples are randomly selected from the sample pool, and we evaluate the performance of each method in terms of F1 score, precision, and accuracy. The entire process is repeated for five times, and we report the average performance together with the standard deviation.

\begin{table}[!t]
\renewcommand{\arraystretch}{1.3} 
\centering
\begin{tabular}{@{} l @{} c   c @{\hspace{5mm}} c @{\hspace{5mm}} c  @{}}

\multicolumn{2}{l}{\textbf{Kronecker}} &   \textbf{F1 Score} & \textbf{Precision} & \textbf{ Accuracy}  \\ 

\midrule

\multicolumn{2}{l}{\textbf{Our Method}}  & 0.994{\small (1E-3)} & 0.995{\small (1E-4)} & 0.993{\small (1E-4)} \\ 

\midrule

\multicolumn{1}{c}{\multirow{5}{*}{\textbf{LR}}} & \textbf{1} & 0.499{\small (1E-3)} & 0.676{\small (2E-3)} &  0.709{\small (2E-3)}  \\

& \textbf{2} & 0.474{\small (1E-3)} & 0.752{\small (1E-2)} &  0.705{\small (4E-3)}  \\ 

& \textbf{3} & 0.460{\small (3E-3)} & 0.802{\small (1E-2)}& 0.705{\small (6E-3)} \\

 & \textbf{4} & 0.452{\small (3E-3)} & 0.827{\small (5E-3)} &  0.701{\small (4E-3)}	\\

& \textbf{5} & 0.451{\small (3E-3)} & 0.832{\small (6E-3)} & 0.701{\small (4E-3)} 	\\ 

\midrule

\multicolumn{1}{c}{\multirow{5}{*}{\textbf{SVM}}} & \textbf{1} & 0.556{\small (2E-3)} & 0.716{\small (2E-3)} &  0.727{\small (3E-3)} 	 \\

& \textbf{2} & 0.496{\small (1E-3)} & 0.731{\small (2E-3)} &  0.705{\small (4E-3)}  \\

& \textbf{3} & 0.480{\small (3E-3)} & 0.770{\small (9E-3)} &  0.702{\small (5E-3)}  \\

& \textbf{4} & 0.470{\small (4E-3)}& 0.798{\small (9E-3)} &  0.697{\small (3E-3)}   \\

& \textbf{5} & 0.465{\small (3E-3)} & 0.818{\small (6E-3)} &  0.698{\small (6E-3)}  	\\ 

\midrule

\multicolumn{1}{c}{\multirow{5}{*}{\textbf{MLP}}} & \textbf{1} & 0.909{\small (1E-3)} & 0.975{\small (1E-4)} &  0.935{\small (1E-4)}   	\\

& \textbf{2} & 0.909{\small (1E-3)} & 0.975{\small (1E-4)} & 0.934{\small (1E-3)}  	\\

& \textbf{3} & 0.909{\small (1E-3)} & 0.975{\small (1E-4)} & 0.935{\small (1E-3)} \\

& \textbf{4} & 0.909{\small (1E-3)} & 0.975{\small (1E-3)} & 0.935{\small (1E-3)}  \\

& \textbf{5} & 0.909{\small (1E-3)} & 0.975{\small (1E-4)} &  0.934{\small (1E-4)}  \\ 

\midrule

\multicolumn{2}{c}{\textbf{Random}}   & 0.209{\small (1E-2)} & 0.250{\small (1E-2)} & 0.250{\small (1E-2)}	\\ 


\end{tabular}%
\caption{Main results on Kronecker graph with 100 training samples. Each cell presents the average performance plus the standard deviation. Testing results for LR, SVM and MLP are provided for different numbers of iterations.}
\label{tab: main}
\vspace{-2mm}
\end{table}

\subsection{Results}

\subsubsection{Major observations}
The main results are given in Table \ref{tab: main}. We see from the table that all the prediction results are statistically robust per the standard deviation. The main observation is that our method outperforms other methods in a significant manner in terms of any of the measurements. MLP is better than LR and SVM, but still produces low-quality predictions on the power-law graph. In addition, for classic learning methods, the Kronecker graph is relatively easy to handle, while the predictions on the power-law graph are consistently less satisfactory, which suggests that they are sensitive to graph structures.   

\subsubsection{Minor observations}

\textbf{Hidden layers meeting diffusion status.} From the standpoint of surrogate modeling, it may be interesting to notice that the hidden features of the proposed neural networks meet exactly the node status during the diffusion process, which somehow makes our model explainable and more transparent. While the original goal is to predict the final diffusion status, we seek to examine to what extent the hidden features are semantically meaningful. To this end, we compare the hidden features with the node status in the corresponding steps, and the results are given in Figure \ref{fig:step}. For our method, we in general observe that the one-step prediction is very accurate, while the match becomes less perfect with more steps. Therefore, given the fact that most of the samples contain no more than six diffusion steps, our method is overall effective. 

\textbf{Imbalance class distribution.} Imbalance class distribution is an inherent issue in predicting the node status in that the majority of the nodes tend to be inactive, and therefore high accuracy does not necessarily imply good prediction performance. According to the results in Table \ref{tab: main}, the standard methods often simply predict all the nodes as inactive, which gives a decent accuracy but low F1 score and precision. In such a sense, our method is more robust and effective.

\textbf{Impact of testing iterations.} For LR and MLP, one can see that repeating the single-step prediction does not help increase prediction quality. One plausible reason is that the single-step prediction is not sufficiently good and its error is amplified by the repetitions. For MLP, the results suggest that the prediction almost converges after the first single-step prediction.

\textbf{Impact of the number of training samples.} Tables \ref{tab: kro_full} and \ref{tab: pl-full} in Appendix \ref{subsec: exp_more} presents the full results on Kronecker and power-law using different training sizes. We see that the performance increases very slightly when more samples are provided, suggesting that training size is not the main bottleneck.

\textbf{Impact of the number of cascades.} As our work focuses on the CLT model, we are interested in the performance under different numbers of cascades, and the results of this part can be found in Tables \ref{tab:kro_234} and \ref{tab: pl_234} in Appendix \ref{subsec: exp_more}. Our method exhibits a stable performance, while for other methods, the task becomes more challenging in the presence of more cascades. The above tables also present the running time of each method, which confirms that our proposed linear programming can be easily handled by standard software packages.

\begin{table}[!t]
\renewcommand{\arraystretch}{1.2} 
\centering
\begin{tabular}{@{} c    c @{\hspace{5mm}} c @{\hspace{5mm}} c  @{}}
\multicolumn{1}{l}{\textbf{Higgs}} &   \textbf{F1 Score} & \textbf{Precision} & \textbf{ Accuracy}  \\ 
\midrule
\multicolumn{1}{l}{\textbf{Our Method}}  &  0.999{\small(1E-4)}&0.999{\small(1E-4)}&0.998{\small(2E-4)}\\ 
\multicolumn{1}{l}{\multirow{1}{*}{\textbf{LR}}} & 0.427{\small(4E-3)} &	0.322{\small(4E-3)} &	0.663{\small(2E-3)}	\\
\multicolumn{1}{l}{\multirow{1}{*}{\textbf{SVM}}} & 0.492{\small(6E-3)}&	0.674{\small(6E-3)}&	0.648{\small(4E-3)}	 \\
\multicolumn{1}{l}{\multirow{1}{*}{\textbf{MLP}}} &   0.939{\small(3E-3)} &	0.962{\small(3E-3)} &	0.960{\small(2E-3)}		\\
\multicolumn{1}{l}{\textbf{Random}}   &  0.451{\small(4E-4)} & 0.450{\small(3E-4)} & 0.450{\small(4E-4)}	  	\\ 


\end{tabular}%
\caption{Main results on Higgs with the cascades using 100 training samples. The testing of LR, SVM, and MLP are performed with one iteration.}
\label{tab: main_higgs}
\end{table}

\begin{figure}[!t]
    \centering
    \includegraphics[width = 0.45\textwidth]{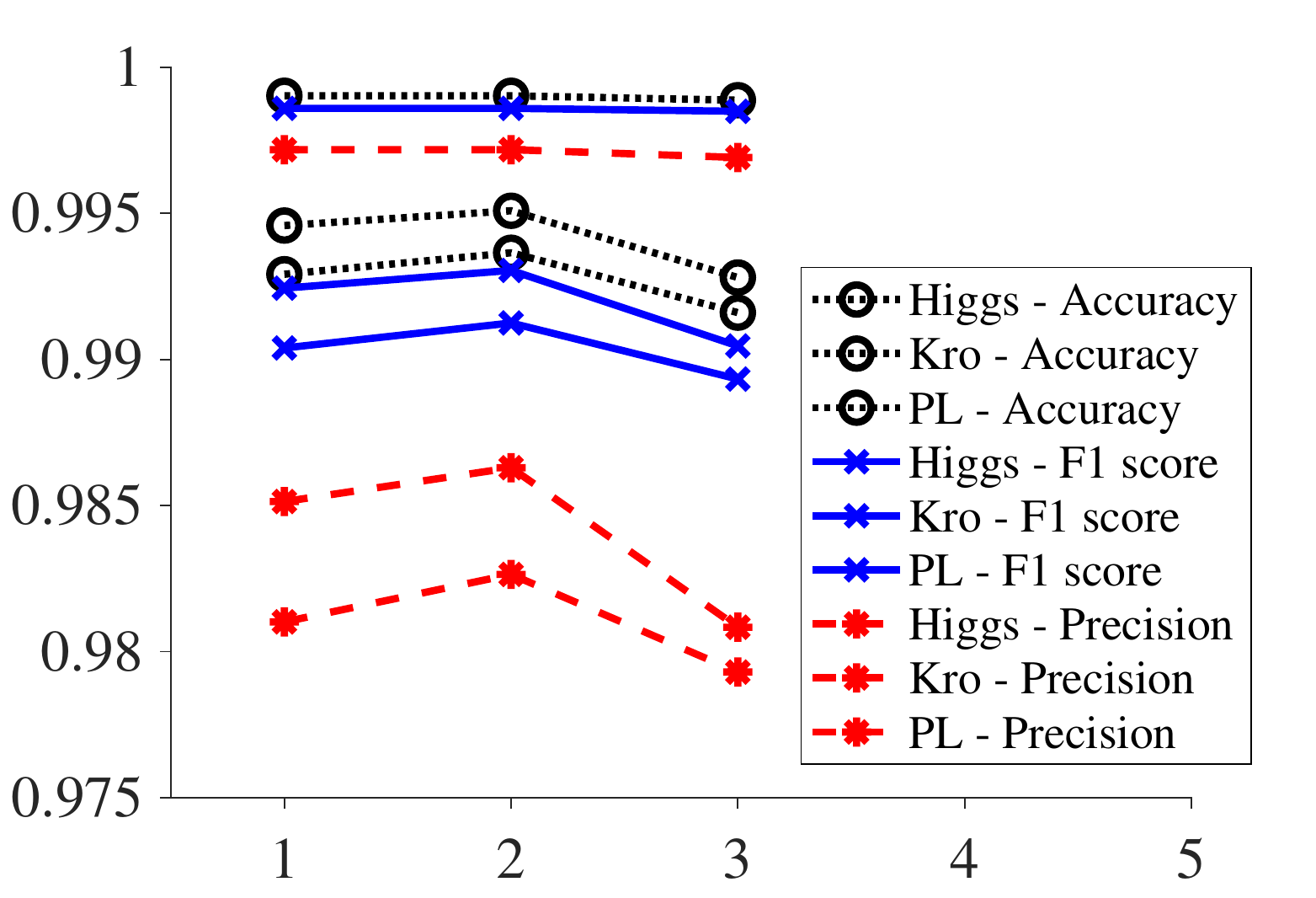}
    \caption{Results of step-step matching. }
    \label{fig:step}
    \vspace{-2mm}
\end{figure}

\section{Further discussions}
In this paper, we demonstrate the possibility of design methods for simulating social contagions by using neural networks that directly mimic the diffusion pattern. We present a learnability study for the competitive threshold models as well as an efficient learning algorithm with a provable sample complexity. While our method outperforms classic supervised learning methods, it does not rule out the possibility that advanced representation techniques can achieve better learning performance with sufficient samples, which is one direction of our future work. In addition, there exist other classic operational models, such as the independent cascade model and voter model, and it remains unknown how to explicitly simulate them through neural networks, or more generally, graphical models. Given the encouraging experimental results we acquire in this paper, we believe it is worth investigating more models beyond the single-cascade case.

\bibliographystyle{named}
\bibliography{ijcai22}
\newpage

\appendix

\section*{Supplementary Materials}
\section{Proofs}

\subsection{Proof of Lemma \ref{lemma:clt_to_nn}}

The numbers of layers, adjustable weights, and computation units follow directly from our design. By scrutiny, all the activation functions are piecewise linear functions. Except for $\sigma^{2Z+3}$, each of the activation functions is with a maximum of two pieces because they are either linear, threshold or ReLu. For $\sigma^{2Z+3}$, the number pieces is bounded by the scale of input, which is $10^Q$ due to the system precision.

\subsection{Proof of Lemma \ref{lemma: efficient_erm}}
Given a collection of samples $D=\{(\mathbf{I}_0^1, \mathbf{I}_{1:N, :, :, 2}^1)\}_{k=1}^K$, 
we prove that the zero error for each data point $(\mathbf{I}_0^k, \mathbf{I}_{1:N, :, :, 2}^k)$ can be achieved by a collection of linear constraints over the parameters. In particular, it is sufficient if the output of every one-step simulation in the neural network matches the ground truth. Fixing a sample and a time step $t>0$, for the simulation of the phase 1 regarding each node $v_i$ and each cascade $s$, the desired matched can be produced by Equation \ref{eq: lp_1_1} (resp, Equation \ref{eq: lp_1_0}) for $\mathbf{I}_{t, i, 1, s }^{k}=1$ (resp, $\mathbf{I}_{t, i, 1, s }^{k}=0$). 
\begin{flalign}
\label{eq: lp_1_1}
&(2(\mathbf{I}_{t, i, s,1 }^{k})-1)(\sum_{v_j\in N(v_i)}w_{j,i}^s\mathbf{I}_{t-1,i, s, 2}^{k}-\theta_{i}^s) \geq 0
\end{flalign}
\begin{flalign}
\label{eq: lp_1_0}
&(2(\mathbf{I}_{t, i, s ,1}^{k})-1)(\sum_{v_j\in N(v_i)}w_{j,i}^s\mathbf{I}_{t-1,i, s,2}^{k}-\theta_{i}^s) > 0
\end{flalign}
For phase 2 and each $s^*$ such that $\mathbf{I}_{t,i, s^*,2}^{k}=1$, the desired matched can be obtained by enforcing the winning cascade has the largest weights among all candidate cascades:
\begin{flalign}
\label{eq: lp_2_1}
&\sum_{v_j\in N(v_i)}w_{j, i}^{s^*}\mathbf{I}_{t,i, s^*,1}^{k} -\sum_{v_j\in N(v_i)}w_{j, i}^s\mathbf{I}_{t,i, s, 1}^{k} \geq 0
\end{flalign}
for each $s> s^*$ and 
\begin{flalign}
\label{eq: lp_2_2}
&\sum_{v_j\in N(v_i)}w_{j, i}^{s^*}\mathbf{I}_{t,i, s^*,1}^{k} -\sum_{v_j\in N(v_i)}w_{j, i}^s\mathbf{I}_{t,i, s, 1}^{k} > 0
\end{flalign}
for each $s< s^*$. Taking the above constraints over all samples and time steps, together with the normalization constraint $\sum_{v_j \in N(v_i)} w_{j,i}^{s}\leq 1$ as well as the non-negative constraint $w_{j,i}\geq 0$, we acquire the entire linear programming.

\subsection{Proof of Lemma \ref{lemma:VC}}

The following standard result will be useful. 

\begin{theorem}{\cite{anthony1999neural}}{\label{theorem:VCofbook}} For a feed-forward network with a total of $a_1 \in \mathbb{Z}$ weights and $ a_2 \in \mathbb{Z}$ computational units, in which the output unit is a linear threshold unit and every other computation unit has a piecewise-polynomial activation function with $a_3 \in \mathbb{Z}$ pieces and degree no more than $a_4 \in \mathbb{Z}$. Suppose in addition that the computation units in the network are arranged in $ a_5 \in \mathbb{Z}$ layers, so that each unit has connections only from units in earlier layers. The VC dimension of the function class computed by such neural networks is upper bounded by
\begin{align*}
2a_1 a_5\log_2(\frac{4a_1 a_2 a_3 a_5}{\ln{2}}) 
                + 2a_1 a_5^2\log_2(a_4 +1)+2a_5
\end{align*}
\end{theorem}

Consider the proposed neural network with $N$ one-step diffusion steps. Applying the above result together with Lemma \ref{lemma:clt_to_nn} to $\F_{i,s}$ yields a VC dimension of $\Tilde{O}(NSQ(N+M))$. In what follows, we prove that tighter bounds are possible. To this end, let us denote by $\F_{i,s}^t$ the class of functions mapping the initial status to the output of after $t$ steps of the simulation in the neural networks, i.e.,
\begin{align*}
\F_{i,s}^t=\Big\{ f_{i,s}^t: i \in [N], s \in [S],  f_{i,s}^t(\mathbf{H}^1)=\mathbf{H}^{1+t(2Z+5)}_{i\cdot S+s} \Big\}.
\end{align*}
We will inductively prove that $\VCdim(\F_{i,s}^t)=  \Tilde{O}(SQ(N+M))$ for each $t$. For the case of $t=1$, we have the desired result due to Lemma \ref{lemma:clt_to_nn} and Theorem \ref{theorem:VCofbook}. Since the parameters are shared in different diffusion steps, the following fact completes the proof, which follows quite directly from Lemma 7 in \cite{narasimhan2015learnability}.

\begin{lemma}\cite{narasimhan2015learnability}
$\VCdim(\F_{i,s}^{t+1})\leq \VCdim(\F_{i,s}^t)$ for each $t>1$.
\end{lemma}

\subsection{Proof of Theorem \ref{theorem: final}}
The following theorem describes the sample complexity for learning a class with a finite VC dimension.
\begin{theorem}{\label{theorem:samplecomplexity}}{\cite{shalev2014understanding}}
Let $\mathcal{H}$ be a hypothesis class of functions from a domain $\mathcal{X}$ to $\{0,1\}$ and let the loss function be the $0$-$1$ loss. Assume that $\VCdim (\mathcal{H}) = d < \infty$. Then, there are absolute constants $C_1, C_2$ such that $\mathcal{H}$ is PAC learnable with sample complexity:
\begin{align*}
    C_1\frac{d+\log (1/\delta)}{\epsilon}\leq m_\mathcal{H}(\epsilon,\delta)\leq C_2\frac{d\log(1/\epsilon)+\log(1/\delta)}{\epsilon}
\end{align*}
in terms of achieving an error no less than $\epsilon$ with probability at least $1-\delta$.
\end{theorem}

Based on Theorem \ref{theorem:samplecomplexity}, the function class $\mathcal{F}_{i,s}$ is PAC learnable by the ERM scheme with a sample complexity of $$O(\frac{\VCdim(\mathcal{F}_{i,s})\log(1/\epsilon) + \log(1/\delta)}{\epsilon}).$$ Taking a union bound over $i \in [N]$ and $s \in S$, the sample complexity for the entire class $\F$ is 
$$O(\frac{\VCdim(\mathcal{F}_{i,s})\log(1/\epsilon) + \log(NS/\delta)}{\epsilon}).$$

\section{Remark \ref{remark: special_case}}
\label{subsec: remark_special_case}
For the scenario where there are two cascades, the coding-recovery scheme for the cascade identity is no longer needed, and the one-step simulation can be implemented by two layers of threshold functions, where the first layer is identical to Equation \ref{eq: H1_H2} and the second layer is a perceptron function, as illustrated in Figure \ref{fig:cascades2}. Therefore, following the similar analysis in Lemma \ref{lemma:VC}, the VC dimension of $\F_{i,s}$ is $\Tilde{O}(M+N)$, without depending on the system precision. 

\begin{figure}[!t]
    \centering
    \includegraphics[width = 0.47\textwidth]{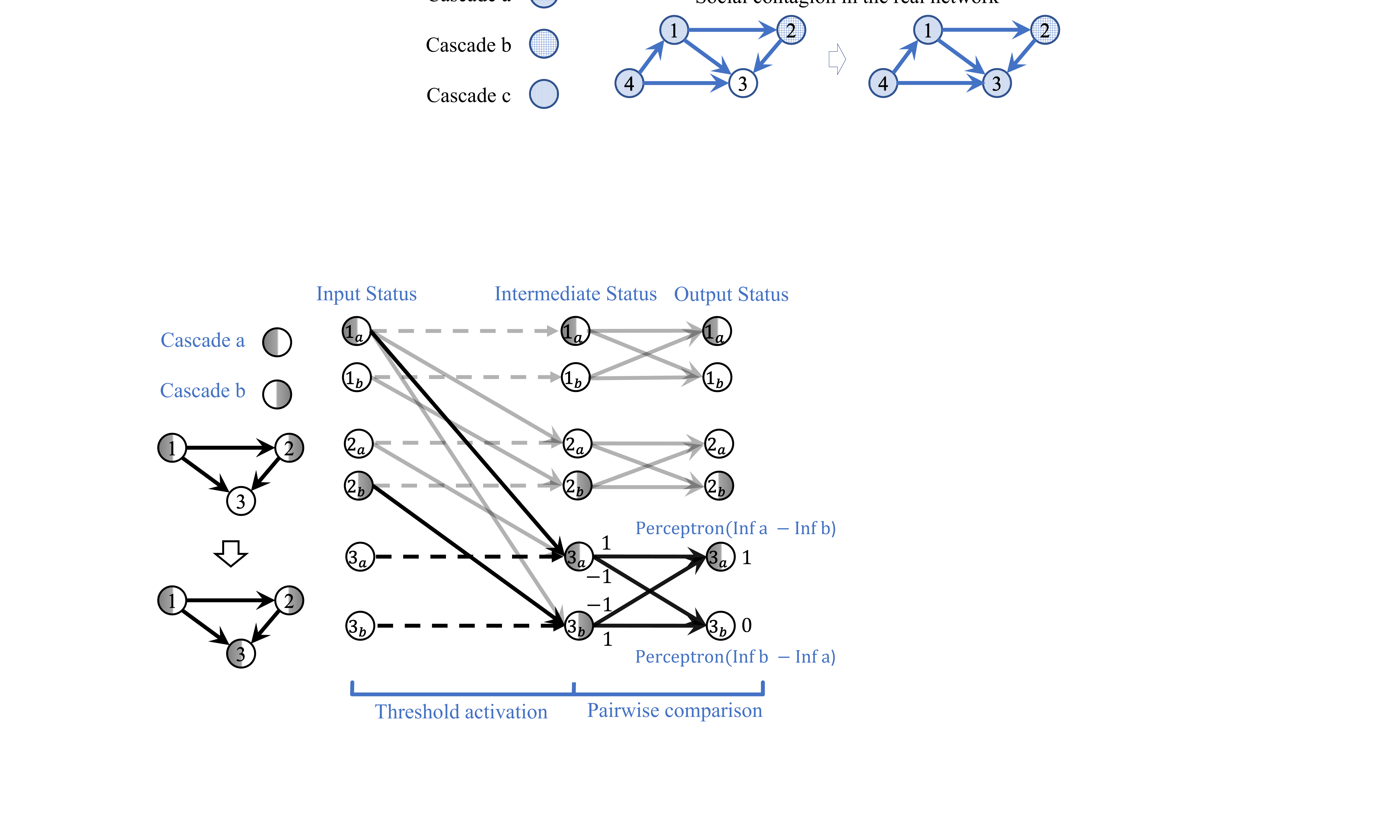}
    \caption{Design for $S=2$.}
    \label{fig:cascades2}
\end{figure}

\section{Additional materials for experiments}{\label{subsec: exp_more}}
The Kronecker graph comes with 1024 nodes and 2655 edges; the \text{power-law} graph has 768 nodes and 2655 edge; the Higgs graph consists of 4632 nodes and 3000 edges. The additional experimental results are given in Tables \ref{tab:kro_234}, \ref{tab: pl_234}, \ref{tab: kro_full} and \ref{tab: pl-full}. The source code, the data regarding the instances (i.e., Kronecker, power-law, and Higgs), and the generated samples are provided in the supplementary materials.

\begin{table}[!t]
\centering
\renewcommand{\arraystretch}{1.2} 
\begin{tabular}{@{} c @{} c   c @{\hspace{2mm}} c @{\hspace{2mm}}  c @{\hspace{2mm}} c    @{}}
& &   \textbf{F1 Score} & \textbf{Precision }&  \textbf{Accuracy} & \textbf{Time ({\tiny min})} \\ 
\midrule
\multicolumn{1}{c}{\multirow{3}{*}{\textbf{\makecell[c]{Our \\ Method}}}} & \textbf{2}  & 0.996{\tiny(3E-4)}&	0.996{\tiny(2E-4)} &	0.995{\tiny(8E-4)} & 2.471
	\\
& \textbf{3} & 0.994{\tiny (1E-3)} & 0.995{\tiny (1E-4)} & 0.993{\tiny (1E-4)} & 3.527
 \\ 
& \textbf{4}  & 0.993{\tiny(3E-4)} &	0.995{\tiny(1E-4)} &	0.994{\tiny(5E-4)} & 4.205
 \\ 
\midrule
\multicolumn{1}{c}{\multirow{3}{*}{\textbf{LR}}} & \textbf{2} & 	0.56{\tiny(3E-3)}& 0.748{\tiny(7E-3)} &	0.717{\tiny(4E-3)}& 375.371
	 \\
& \textbf{3} & 0.909{\tiny (1E-3)} & 0.975{\tiny (1E-4)} &  0.934{\tiny (1E-4)} & 372.102
 \\ 
& \textbf{4} & 0.446{\tiny(2E-3)} &	0.716{\tiny(6E-3)}	& 0.709{\tiny(5E-3)} & 370.016
\\
\midrule
\multicolumn{1}{c}{\multirow{3}{*}{\textbf{SVM}}} & \textbf{2 }
& 0.576{\tiny(1E-3)} &	0.751{\tiny(4E-3)}	& 0.724{\tiny(1E-3)}& 382.598
	\\
& \textbf{3}  & 0.465{\tiny (3E-3)} & 0.818{\tiny (6E-3)} &  0.698{\tiny (6E-3)} & 390.132
 \\ 
& \textbf{4} & 0.473{\tiny(2E-3)} &	0.722{\tiny(8E-3)} &	0.717{\tiny(3E-3)}	& 402.579
\\
\midrule
\multicolumn{1}{c}{\multirow{3}{*}{\textbf{MLP}}} & \textbf{2} 
& 0.566{\tiny(1E-3)} &	0.672{\tiny(3E-3)} &	0.709{\tiny(3E-3)}& 183.860
	\\
& \textbf{3}  & 0.909{\tiny (1E-3)} & 0.975{\tiny (1E-4)} &  0.934{\tiny (1E-4)} & 126.761
 \\ 
& \textbf{4} & 0.456{\tiny(1E-3)} &	0.57{\tiny(5E-3)} &	0.691{\tiny(5E-3)}& 140.503
	\\

\end{tabular}%
\caption{The results on Kronecker graph with different numbers of the cascades using 100 training samples.}
\label{tab:kro_234}
\end{table}

\begin{table}[!h]
\centering
\renewcommand{\arraystretch}{1.2} 
\begin{tabular}{@{} c @{} c   c @{\hspace{2mm}} c @{\hspace{2mm}}  c @{\hspace{2mm}} c  @{}}
& &   \textbf{F1 Score} & \textbf{Precision }&  \textbf{Accuracy }& \textbf{Time ({\tiny min})} \\ 

\midrule
\multicolumn{1}{c}{\multirow{3}{*}{\textbf{\makecell[c]{Our \\ Method}}}} & \textbf{2} & 0.990{\tiny(4E-4)} &	0.994{\tiny(4E-4)} &	0.990{\tiny(1E-3)} &  2.137

	\\
& \textbf{3} & 0.991\tiny (1E-3)	& 0.993\tiny (1E-4) &	0.902\tiny (1E-3) & 3.102 
 \\ 
& \textbf{4} & 0.983{\tiny(5E-4)}&	0.992{\tiny(2E-4)} &	0.990{\tiny(4E-4)} &	3.450 
 \\ 

\midrule

\multicolumn{1}{c}{\multirow{3}{*}{\textbf{LR}}} & \textbf{2} & 0.556{\tiny(3E-3)} &	0.691{\tiny(9E-3)} &	0.708{\tiny(4E-3)}  & 343.654
	 \\
& \textbf{3} & 0.909\tiny (1E-4) &	0.975\tiny (1E-4) &	0.934\tiny (1E-3) &  339.708 
 \\ 
& \textbf{4} & 0.436{\tiny(2E-3)} &	0.624{\tiny(8E-3)}&	0.698{\tiny(4E-3)} & 337.804
\\

\midrule

\multicolumn{1}{c}{\multirow{3}{*}{\textbf{SVM}}} & \textbf{2} & 0.568{\tiny(2E-3)}&	0.703{\tiny(1E-2)}&	0.716{\tiny(3E-3)} &343.504
	\\
& \textbf{3} &  0.465\tiny (4E-3) &	0.757\tiny (2E-2) &	0.692\tiny (5E-3) & 350.709 
 \\ 
& \textbf{4} &0.462{\tiny(3E-3)} &	0.663{\tiny(1E-2)}&	0.707{\tiny(6E-3)} & 354.969
\\
\midrule

\multicolumn{1}{c}{\multirow{2}{*}{\textbf{MLP}}} & \textbf{2} &0.560{\tiny(2E-3)} &	0.638{\tiny(7E-3)} &	0.700{\tiny(4E-3)} & {161.116}
	\\
& \textbf{3} & 0.909\tiny (1E-4) &	0.975\tiny (1E-4) &	0.934\tiny (1E-3) & 168.902 
 \\ 
& \textbf{4} &0.440{\tiny(1E-3)} &	0.530{\tiny(5E-3)} &	0.677{\tiny(4E-3)} &177.176
	\\


\end{tabular}%
\caption{The results on power-law graph with different numbers of the cascades using 100 training samples.}
\label{tab: pl_234}
\end{table}


\begin{table*}[t]
\renewcommand{\arraystretch}{1.2} 
\begin{tabular}{@{} c @{} c   c @{\hspace{2mm}} c @{\hspace{2mm}} c @{\hspace{5mm}} c @{\hspace{2mm}} c @{\hspace{2mm}}c  @{\hspace{5mm}} c  @{\hspace{2mm}} c  @{\hspace{2mm}} c @{}}
& & \multicolumn{3}{c}{\textbf{50}} & \multicolumn{3}{c}{\textbf{100}} & \multicolumn{3}{c}{\textbf{500}} \\ 
&  & \textbf{F1 Score} & \textbf{Precision} &  \textbf{Accuracy} & \textbf{F1 Score} & \textbf{Precision} & \textbf{ Accuracy } & \textbf{F1 Score} & \textbf{Precision} & \textbf{ Accuracy }\\ 

\midrule

\multicolumn{2}{c}{\textbf{Our Method}}  & 0.991\tiny (1E-3)	& 0.993\tiny (1E-4) &	0.902\tiny (1E-3) & 0.994{\tiny (1E-3)} & 0.995{\tiny (1E-4)} & 0.993{\tiny (1E-4)}  & 1.000\tiny (1E-4) & 0.999\tiny (1E-4) &  0.999\tiny (1E-4) \\ 

\midrule

\multirow{5}{*}{\textbf{LR}} & \textbf{1} & 0.488	\tiny (2E-3) &	0.622\tiny (2E-3) &		0.695\tiny (2E-3) &   0.499{\tiny (1E-3)} & 0.676{\tiny (2E-3)} &  0.709{\tiny (2E-3)}  & 0.509\tiny (1E-3) & 0.732\tiny (1E-3) &  0.719\tiny (4E-3) \\

& \textbf{2} & 0.471\tiny (2E-3) &	0.668\tiny (1E-2)&  0.697\tiny (4E-3) &  0.474{\tiny (1E-3)} & 0.752{\tiny (1E-2)} &  0.705{\tiny (4E-3)} & 0.473\tiny (2E-3) & 0.811\tiny (4E-3) &  0.708\tiny (3E-3) \\

& \textbf{3} & 0.460\tiny (1E-3) &	0.713\tiny (2E-2)&	0.697\tiny (4E-3) & 0.460{\tiny (3E-3)} & 0.802{\tiny (1E-2)}& 0.705{\tiny (6E-3)} & 0.454\tiny (1E-3) & 0.852\tiny (3E-3) &  0.703\tiny (5E-3) \\

& \textbf{4} &  0.458\tiny (4E-3)&	0.743\tiny (2E-2)&	0.698\tiny 61E-3) &  0.452{\tiny (3E-3)} & 0.827{\tiny (5E-3)} &  0.701{\tiny (4E-3)} & 0.449\tiny (1E-3) & 0.867\tiny (4E-3) & 0.703\tiny (5E-3) \\

& \textbf{5} & 0.454	\tiny (5E-3)&	0.772\tiny (2E-2)&	0.699\tiny (4E-3) &  0.451{\tiny (3E-3)} & 0.832{\tiny (6E-3)} & 0.701{\tiny (4E-3)}  & 0.446\tiny (2E-3) & 0.876\tiny (3E-3) &  0.703\tiny (3E-3) \\ 

\midrule

\multirow{5}{*}{\textbf{SVM}} & \textbf{1} &0.537\tiny (1E-3)&	0.663	\tiny (8E-3)&		0.716\tiny (4E-3) & 0.556{\tiny (2E-3)} & 0.716{\tiny (2E-3)} &  0.727{\tiny (3E-3)}  & 0.583\tiny (1E-3) & 0.797\tiny (2E-3)2 &  0.747\tiny (3E-3) \\

& \textbf{2} & 0.488\tiny (1E-3) &	0.674\tiny (5E-3) &	0.695\tiny (2E-3) &   0.496{\tiny (1E-3)} & 0.731{\tiny (2E-3)} &  0.705{\tiny (4E-3)} & 0.504\tiny (2E-3) & 0.805\tiny (4E-3) &  0.713\tiny (6E-3) \\

& \textbf{3} & 0.476\tiny (3E-3) &	0.705\tiny (1E-2) &	0.693\tiny (2E-3) &  0.480{\tiny (3E-3)} & 0.770{\tiny (9E-3)} &  0.702{\tiny (5E-3)}  & 0.483\tiny (2E-3) & 0.834\tiny (1E-3) &  0.706\tiny (4E-3) \\

& \textbf{4} & 0.471\tiny 4E-3) &	0.734\tiny (2E-2) &	0.696\tiny (2E-3) &  0.470{\tiny (4E-3)}& 0.798{\tiny (9E-3)} &  0.697{\tiny (3E-3)}  & 0.473\tiny (1E-3) & 0.855\tiny (2E-3) &  0.701\tiny (4E-3) \\

& \textbf{5} & 0.465\tiny (4E-3) &	0.757\tiny (2E-2) &	0.692\tiny (5E-3) &  0.465{\tiny (3E-3)} & 0.818{\tiny (6E-3)} &  0.698{\tiny (6E-3)}& 0.468\tiny (3E-3) & 0.863\tiny (2E-3) &  0.700\tiny (6E-3) \\ 

\midrule

\multirow{5}{*}{\textbf{MLP}} & \textbf{1} & 0.909\tiny (1E-3)&	0.975\tiny (1E-4)&	0.935\tiny (1E-3) &  0.909{\tiny (1E-3)} & 0.975{\tiny (1E-4)} &  0.935{\tiny (1E-4)} & 0.909\tiny (1E-3) & 0.975\tiny (1E-4) & 0.935\tiny (1E-4) \\

& \textbf{2} & 0.909\tiny (1E-3) &	0.975\tiny (1E-3) &	0.934\tiny (1E-3) &  0.909{\tiny (1E-3)} & 0.975{\tiny (1E-4)} & 0.934{\tiny (1E-3)} & 0.908\tiny (1E-3) & 0.975\tiny (1E-4) & 0.935\tiny (1E-3) \\

& \textbf{3} & 0.909\tiny (1E-3) &	0.975\tiny (1E-4) &	0.934\tiny (1E-3) & 0.909{\tiny (1E-3)} & 0.975{\tiny (1E-4)} & 0.935{\tiny (1E-3)}  & 0.909\tiny (1E-4) & 0.975\tiny (1E-4) & 0.934\tiny (1E-3) \\

& \textbf{4} & 0.909\tiny (1E-4) &	0.975\tiny (1E-4) &	0.935\tiny (1E-3) &  0.909{\tiny (1E-3)} & 0.975{\tiny (1E-3)} & 0.935{\tiny (1E-3)} & 0.909\tiny (1E-3) & 0.975\tiny (1E-4) & 0.934\tiny (1E-3) \\

& \textbf{5} & 0.909\tiny (1E-4) &	0.975\tiny (1E-4) &	0.934\tiny (1E-3) &  0.909{\tiny (1E-3)} & 0.975{\tiny (1E-4)} &  0.934{\tiny (1E-4)} & 0.909\tiny (1E-4) & 0.975\tiny (1E-4) & 0.935\tiny (1E-3) \\ 

\midrule

\multicolumn{2}{c}{\textbf{Random}}  & 0.208\tiny (2E-3)	& 0.250\tiny (2E-3) &	0.250\tiny (3E-3) &	 0.209{\tiny (1E-2)} & 0.250{\tiny (1E-2)} & 0.250{\tiny (1E-2)} & 0.210\tiny (1E-3) & 0.250\tiny (1E-3)  & 0.250\tiny (1E-3) \\ 
\end{tabular}%

\caption{The results on Kronecker graph with three cascades and different training sizes.}
\label{tab: kro_full}
\end{table*}

\begin{table*}[h]
\renewcommand{\arraystretch}{1.2} 
\begin{tabular}{@{} c @{} c   c @{\hspace{2mm}} c @{\hspace{2mm}} c @{\hspace{5mm}} c @{\hspace{2mm}} c @{\hspace{2mm}}c  @{\hspace{5mm}} c  @{\hspace{2mm}} c  @{\hspace{2mm}} c @{}}
& & \multicolumn{3}{c}{\textbf{50}} & \multicolumn{3}{c}{\textbf{100}} & \multicolumn{3}{c}{\textbf{500}} \\ 
&  & \textbf{F1 Score} & \textbf{Precision} &  \textbf{Accuracy} & \textbf{F1 Score} & \textbf{Precision} & \textbf{ Accuracy } & \textbf{F1 Score} & \textbf{Precision} & \textbf{ Accuracy }\\ 

\midrule

\multicolumn{2}{c}{\textbf{Our Method}} &  0.987\tiny (1E-3)&	0.993\tiny (1E-4)&	0.990\tiny (1E-3)&	0.994\tiny (1E-3)&	0.99\tiny (1E-4)&	0.994\tiny (1E-4)&	1.000\tiny (1E-4)&	0.998\tiny (1E-4)&	0.999\tiny (1E-4)\\

\midrule

\multirow{5}{*}{\textbf{LR}} & \textbf{1} &  0.207\tiny (1E-3) &	0.295\tiny (3E-3)&		0.606\tiny (4E-3)&	0.207\tiny (2E-3)&	0.291\tiny (9E-3)&		0.618\tiny (1E-2)&	0.207\tiny (3E-3)&	0.288\tiny (8E-3)&		0.619\tiny (3E-3)\\
& \textbf{2} &  0.209\tiny (4E-3)&	0.291\tiny (3E-3)&		0.601\tiny (1E-4)0&	0.208\tiny (5E-3)&	0.283\tiny (5E-3)&		0.609\tiny (2E-3)&	0.205\tiny (1E-3)&	0.293\tiny (4E-3)&	0.609\tiny (9E-3)	\\
& \textbf{3} & 0.211\tiny (2E-3)&	0.289\tiny (7E-3)&		0.613\tiny (9E-3)&	0.209\tiny (1E-3)&	0.293\tiny (1E-2)&		0.609\tiny (6E-3)&	0.206\tiny (2E-3)&	0.291\tiny (1E-2)&	0.613\tiny (2E-3)	\\
& \textbf{4} & 0.211\tiny (2E-3)&	0.304\tiny (3E-3)&	0.612\tiny (4E-3)&	0.205\tiny (4E-3)&	0.294\tiny (9E-3)&		0.611\tiny (1E-2)&	0.206\tiny (1E-4)&	0.293\tiny (3E-3)&		0.609\tiny (9E-3)	\\
& \textbf{5} & 0.208\tiny (6E-3)&	0.287\tiny (1E-2)&		0.604\tiny (3E-3)&	0.206\tiny (1E-4)&	0.284\tiny (9E-3)&		0.608\tiny (1E-3)&	0.206\tiny (1E-4)&	0.285\tiny (7E-3)&	0.611\tiny (6E-3)\\

\midrule

\multirow{5}{*}{\textbf{SVM}} & \textbf{1} & 0.231\tiny (4E-3)&	0.279\tiny (3E-3)&	0.617\tiny (1E-3)&	0.230\tiny (2E-3)&	0.281\tiny (1E-2)&	0.622\tiny (3E-3)&	0.228\tiny (2E-3)&	0.257\tiny (2E-3)&	0.626\tiny (7E-3)\\
& \textbf{2} & 0.233\tiny (2E-3)&	0.290\tiny 9E-3)&	0.611\tiny (2E-3)&	0.231\tiny (2E-3)&	0.286\tiny (5E-3)&	0.618\tiny (1E-3)&	0.229\tiny (1E-3)&	0.258\tiny (2E-3)&	0.624\tiny (1E-3)\\
& \textbf{3} & 0.230\tiny (5E-3)&	0.275\tiny (7E-3)&  0.626\tiny (5E-3)&	0.226\tiny (2E-3)&	0.272\tiny (3E-3)&	0.617\tiny (6E-3)&	0.229\tiny (3E-3)&	0.263\tiny (1E-3)&	0.620\tiny (1E-2)\\
& \textbf{4} & 0.229\tiny (1E-3)&	0.283\tiny (4E-3)&	0.616\tiny (6E-3)&	0.230\tiny (1E-4)&	0.278\tiny (4E-3)&	0.617\tiny (1E-3)&	0.231\tiny (1E-4)&	0.268\tiny (4E-3)&	0.623\tiny (5E-3)\\
& \textbf{5} & 0.232\tiny (1E-3)&	0.294\tiny (1E-2)&	0.620\tiny (5E-3)&	0.231\tiny (1E-3)&	0.280\tiny (1E-3)&	0.620\tiny (1E-2)&	0.230\tiny (1E-3)&	0.265\tiny (4E-3)&	0.619\tiny (3E-3)\\

\midrule

\multirow{5}{*}{\textbf{MLP}} & \textbf{1} & 0.235\tiny (1E-4)&	0.273\tiny (8E-3)&	0.613\tiny (5E-3)&	0.230\tiny (1E-3)&	0.272\tiny (6E-3)&	0.621\tiny (1E-4)&	0.231\tiny (4E-3)&	0.268\tiny (1E-2)&	0.621\tiny (3E-3)\\

& \textbf{2} & 0.233\tiny (2E-3)&	0.271\tiny (6E-3)&	0.609\tiny (6E-3)&	0.233\tiny (2E-3)&	0.286\tiny (2E-3)&	0.628\tiny (1E-2)&	0.231\tiny (4E-3)&	0.269\tiny (1E-3)&	0.629\tiny (3E-3)\\

& \textbf{3} & 0.234\tiny (2E-3)&	0.270\tiny (1E-3)&	0.611\tiny (1E-4)&	0.230\tiny (3E-3)&	0.272\tiny (1E-2)&	0.622\tiny (7E-3)&	0.230\tiny (2E-3)&	0.273\tiny (1E-4)&	0.623\tiny (4E-3)\\

& \textbf{4} & 0.231\tiny (1E-3)&	0.264\tiny (1E-4)&	0.614\tiny (1E-2)&	0.227\tiny (5E-3)&	0.274\tiny (4E-3)&	0.617\tiny (2E-3)&	0.229\tiny (2E-3)&	0.284\tiny (1E-2)&	0.621\tiny (1E-3)\\

& \textbf{5} & 0.234\tiny (3E-3)&	0.268\tiny (4E-3)&	0.608\tiny (1E-2)&	0.230\tiny (1E-3)&	0.280\tiny (1E-2)&	0.619\tiny (4E-3)&	0.228\tiny (1E-3)&	0.261\tiny (2E-3)&	0.620\tiny (4E-3)\\

\midrule

\multicolumn{2}{c}{\textbf{Random}}  &  0.217\tiny (1E-3)&	0.250\tiny (2E-3)&	0.249\tiny (1E-3)&	0.214\tiny (2E-3)2&	0.251\tiny (2E-3)&	0.251\tiny (1E-3)&	0.213\tiny (1E-4)&	0.25\tiny (1E-4)&	0.250\tiny (1E-4)\\

\end{tabular}%

\caption{The results on Power-law graph with three cascades and different training sizes.}
\label{tab: pl-full}
\end{table*}

\vfill
\clearpage

\end{document}